\documentclass[final,12pt]{msml2020} 

\usepackage{bm}
\newcommand{\comment}[1]{}

\usepackage[T1]{fontenc}
\usepackage[utf8]{inputenc}
\usepackage{caption}
\usepackage{subcaption}

\usepackage{microtype}

\makeatletter

\renewcommand{\vec}[1]{\bm{#1}}
\newcommand{\D}{\mathrm{d}}

\newcommand{\T}{\intercal}
\newcommand{\Exp}{\mathbb{E}}

\newcommand*\diff{\mathop{}\!\D}

\newcommand{\abs}[1]{\lvert#1\rvert} 
\newcommand{\Abs}[1]{\left\lvert#1\right\rvert} 
\newcommand{\norm}[1]{\lVert#1\rVert}

\newcommand{\vzero}{\bm{0}}

\newcommand{\vtheta}{\bm{\theta}}

\newcommand{\vxi}{\bm{\xi}}

\newcommand{\vTheta}{\bm{\Theta}}

\newcommand{\vb}{\bm{b}}

\newcommand{\vg}{\bm{g}}
\newcommand{\vh}{\bm{h}}

\newcommand{\vk}{\bm{k}}

\newcommand{\vs}{\bm{s}}

\newcommand{\vx}{\bm{x}}

\newcommand{\vz}{\bm{z}}

\newcommand{\vY}{\bm{Y}}

\newcommand{\mzero}{\bm{0}}

\newcommand{\mSigma}{\bm{\Sigma}}

\newcommand{\mU}{\bm{U}}
\newcommand{\mV}{\bm{V}}
\newcommand{\mW}{\bm{W}}
\newcommand{\mX}{\bm{X}}

\DeclareMathAlphabet{\mathsfit}{\encodingdefault}{\sfdefault}{m}{sl}
\SetMathAlphabet{\mathsfit}{bold}{\encodingdefault}{\sfdefault}{bx}{n}

\newcommand{\sR}{\mathbb{R}}

\newcommand{\RS}{R_S}

\newcommand{\dist}{\mathrm{dist}}

 \newtheorem{thm}{Theorem}
 \newtheorem{cor}{Corollary}
 
 \newtheorem{lem}{Lemma}

 \newtheorem*{rmk*}{\protect Remark}

\usepackage{url}
\usepackage{booktabs}
\usepackage{amsfonts}
\usepackage{nicefrac}

\usepackage{xcolor}

\title[generalization error induced by initialization in DNN]{A type of generalization error induced by initialization  in deep neural networks}
\usepackage{times}

\msmlauthor{ 
\Name{Yaoyu Zhang} \Email{yaoyu@ias.edu}\\
 \addr School of Mathematics, Institute for Advanced Study, Princeton, NJ 08540, USA
 \AND
 \Name{Zhi-Qin John Xu} \Email{xuzhiqin@sjtu.edu.cn\thanks{Corresponding author}}\\
 \addr School of Mathematical Sciences and Institute of Natural Sciences, Shanghai Jiao Tong University, Shanghai, China,
 \AND
 \Name{Tao Luo} \Email{luo196@purdue.edu}\\
 \addr Department of Mathematics, Purdue University, West Lafayette, IN
47907, USA
 \AND 
 \Name{Zheng Ma} \Email{ma531@purdue.edu}\\
 \addr Department of Mathematics, Purdue University, West Lafayette, IN
47907, USA }

\begin{document}

\maketitle

\begin{abstract}%
  How initialization and loss function affect the learning of a deep neural network (DNN), specifically its generalization error, is an important problem in practice. In this work, by exploiting the linearity of DNN training dynamics in the NTK regime  \citep{jacot2018neural,lee2019wide}, we provide an explicit and quantitative answer to this problem. Focusing on regression problem, we prove that, in the NTK regime, for any loss in a general class of functions, the DNN finds the same \emph{global} minima---the one that is nearest to the initial value in the parameter space, or equivalently, the one that is closest to the initial DNN output in the corresponding reproducing kernel Hilbert space. Using these optimization problems, we quantify the impact of initial output and prove that a random non-zero one increases the generalization error. We further propose an antisymmetrical initialization (ASI) trick that eliminates this type of error and accelerates the training. To understand whether the above results hold in general, we also perform experiments for DNNs in the non-NTK regime, which demonstrate the effectiveness of our theoretical results and the ASI trick in a qualitative sense. Overall, our work serves as a baseline for the further investigation of the impact of initialization and loss function on the generalization of DNNs, which can potentially guide and improve the training of DNNs in practice.
\end{abstract}


\section{Introduction\label{sec:Introduction}}

The wide application of deep learning makes it increasingly urgent
to establish quantitative theoretical understanding of the learning
and generalization behaviors of deep neural networks (DNNs). In this
work, we study theoretically the problem of how initialization and
loss function quantitatively affect these behaviors of DNNs. Our study
focuses on the regression problem, which plays a key role in many
applications, e.g., simulation of physical systems \citep{zhang2018deep},
prediction of time series \citep{qiu2014ensemble} and solving differential
equations \citep{weinan2018deep,xu2019frequency}. For theoretical
analysis, we consider training dynamics of sufficiently-wide DNNs in the NTK (Neural Tanget Kernel) regime, where they can be well approximated by linear gradient flows resembling kernel methods \citep{jacot2018neural,lee2019wide, arora2019fine,arora2019exact}. Note that, theoretical investigation of such a regime
can provide insight into the understanding of general DNNs in practice
by the following facts. Heavy overparameterization is one of the key
empirical tricks to overcome the learning difficulty of DNNs \citep{zhang2016understanding}.
DNNs in the extremely over-parameterized regime preserve substantive
behavior as those in mildly over-parameterized regime. For example,
stochastic gradient descent (SGD) can find global minima of the training
objective of DNNs which generalizes well to the unseen data \citep{zhang2016understanding}. 

In general, the error of DNN can be classified into three general
types \citep{poggio2018theory}: approximation error induced by the
capacity of the hypothesis set, generalization error induced by the
given training data, and training error induced by the given training
algorithm. By the universal approximation theorem \citep{cybenko1989approximation}
and empirical experiments \citep{zhang2016understanding}, a large
neural network often has the power to express functions of real datasets
(small approximation error) and the gradient-based training often
can find global minima (zero training error). Generalization
error is the main source of error in applications. It can be affected
by many factors, such as initialization and loss function as widely
observed in experiments. Empirically, a large weight initialization
often leads to a large generalization error \citep{xu_training_2018,xu2019frequency}.
However, a too small weight initialization makes the training extremely
slow. Note that zero initialization leads to a saddle point of DNN
which makes the training impossible. Despite above empirically observations,
it remains unclear how initialization is related to the generalization
error. Regarding the loss function, it is also unclear how it affects
the behavior of DNNs.

Technically, all our theoretical results are natural consequences of the linearity of NTK dynamics. Our major contribution lies in exploiting this linearity of DNNs in the NTK regime to answer explicitly and quantitatively the important question of how loss function and initialization affect generalization. As demonstrated by experiments, our quantitative results also qualitatively hold for DNNs in the non-NTK regime, thus shedding light into the interaction between loss, initialization and generalization for general DNNs.

Our key results are summarized as follows.

i) We prove that, for a general class of loss functions, the NTK gradient flow, despite trajectory difference, finds the same \emph{global}
minimum.
\comment{
ii) We prove the equivalence among problems of NTK gradient flow, minimum Euclidean norm to initialization in parameter space and minimum kernel norm to initial DNN output $h_{\mathrm{ini}}$ in the corresponding
reproducing kernel Hilbert space (RKHS). 
}

ii) We prove that the bias induced by a nonzero $h_{\mathrm{ini}}$ (initial DNN output) equals to the residual of $h_{\mathrm{ini}}$ trained by the same DNN on the same training inputs initialized with $h_{\mathrm{ini}}=0$. We also show by experiments that this equality approximately holds in non-NTK regime.

iii) We prove that a random $h_{\mathrm{ini}}$ leads to a specific type of generalization error that worsens generalization.

iv) We propose an AntiSymmetrical Initialization (ASI) trick for general DNNs that can offset any $h_{\mathrm{ini}}$ to zero while keeping the NTK at initialization unchanged. Demonstrated by experiments using general DNNs, this trick accelerates the training and improves generalization.

\section{Related works}

The impact of loss and initialization for classical linear regression problem has been discussed in \citet{pmlr-v80-gunasekar18a}. However, such results for deep neural networks are absent. Based on the recent discovery of the NTK regime of DNNs, our work provides quantitative results about loss and initialization, which also sheds light into general DNNs in the non-NTK regime.

There are a series of research on the  NTK regime of DNNs \citep{jacot2018neural,arora2019exact,du2018gradient,zou2018stochastic,allen2018convergence,E2019analysis,E2019comparative,sankararaman2019impact,arora2019fine,cao2019generalization}. Previous works found that  the learning
of DNNs in the NTK regime is equivalent to kernel ridge regression, such as \citet{mei2019mean,banburski2019theory}, and  the convergence point in the parameter space of a DNN remains close to the initialization, such as \citet{chizat2018note,oymak2018overparameterized,jacot2018neural}. For the completeness of the paper, we present proofs of the equivalence among problems of NTK gradient flow, minimum Euclidean norm to initialization in parameter space and minimum kernel norm to initial DNN output $h_{\mathrm{ini}}$ in the corresponding
reproducing kernel Hilbert space (RKHS), which are the foundation of our results. 

Previous works \citep{xu_training_2018,xu2018understanding,xu2018frequency,xu2019frequency,rahaman2018spectral}
discover a Frequency-Principle (F-Principle) that DNNs prefer to learn
the training data by a low-frequency function. Based on F-principle,
\citet{xu_training_2018,xu2019frequency} postulate that the final
output of a DNN tends to inherit high frequencies of its initial output
that can not be well constrained by the training data \citep{xu_training_2018,xu2019frequency}.
Our theoretical and empirical results on initialization provide justification of this postulate.

\section{Preliminary}

\subsection{Notations}
$\Omega:$ a compact domain of $\sR^{d}$; $d$: dimension
of input of DNN; $f$: target function, $f\in L^{\infty}(\Omega)$;
$M$: number of parameter of DNN; $n$: number of training
samples; $\mX$: inputs of training set $(\vx_{1},\cdots,\vx_{n})^{\T}\in\sR^{n\times d}$;
$\vY$: outputs of training set $(y_{1},\cdots,y_{n})^{\T}\in\sR^{n}$;
$\vg(\mX)$: $(g(\vx_{1}),\cdots,g(\vx_{n}))^{\T}$ for any function $g$ on $\Omega$; $h(\vx,\vtheta)$:
output of DNN of parameters $\vtheta$ at $\vx$; $\nabla_{\vtheta}(\cdot)$:
$(\partial_{\theta_1}(\cdot),\cdots,\partial_{\theta_m}(\cdot))^{\T}$;
$\dist$: a general differentiable loss function satisfying conditions
(i)--(iii) in Section \ref{subsec:Kernel-regime}; $k(\cdot,\cdot)$:
kernel function defined as $\nabla_{\vtheta}h(\cdot,\vtheta_{0})^{\T}\nabla_{\vtheta}h(\cdot,\vtheta_{0})$
if there is no ambiguity; $k_{\vtheta'}(\cdot,\cdot)$: kernel function
defined as $\nabla_{\vtheta}h(\cdot,\vtheta')^{\T}\nabla_{\vtheta}h(\cdot,\vtheta')$;
$H_{k}(\Omega)$: reproducing kernel Hilbert space (RKHS)
with respect to kernel $k$ at domain $\Omega$. $\left\langle \cdot,\cdot\right\rangle _{k}$:
inner product of space $H_{k}(\Omega)$; $\norm{\cdot}_{k}$:
norm of space $H_{k}(\Omega)$; $h_{k}(\vx;h_\mathrm{ini},\mX,\vY)$:
the solution of problem \eqref{eq: problem min kernel norm} depending
on kernel $k$, initial function $h_\mathrm{ini}$, inputs $\mX$ and
outputs $\vY$ of training set. 

\subsection{NTK regime of DNN\label{subsec:Kernel-regime}}

In the following, we consider the regression problem of fitting the
target function $f\in L^{\infty}(\Omega)$, where $\Omega$ is a compact
domain in $\sR^{d}$. Clearly, $f\in L^{p}(\Omega)$ for $1\leq p\leq \infty$. Specifically,
we use a DNN, $h(\vx,\vtheta(t)):\Omega\times\sR^{m}\to\sR$,
to fit the training dataset $\{(\vx_{i},y_{i})\}_{i=1}^{n}$ of $n$ sampling
points, where $\vx_{i}\in\Omega$, $y_{i}=f(\vx_{i})$ for each $i$.
For the convenience of notation, we denote $\mX=(\vx_{1},\cdots,\vx_{n})^{\T}$,
$\vY=(y_{1},\cdots,y_{n})^{\T}$, and $\vg(\mX):=(g(\vx_{1}),\cdots,g(\vx_{n}))^{\T}$ for any function $g$ defined on $\Omega$.

The NTK regime refers to a state of DNN that its NTK defined as  
\begin{equation*}
k(\cdot,\cdot)=\nabla_{\vtheta}h(\cdot,\vtheta(t))^{\T}\nabla_{\vtheta}h(\cdot,\vtheta(t))
\end{equation*}
almost does not change throughout the training \citep{jacot2018neural,chizat2018note,arora2019exact}. Note that,
we have the following requirements for $h$ which are
easily satisfied for common DNNs: For any $\vtheta\in\sR^{M}$,
there exists a weak derivative of $h\left(\cdot,\vtheta\right)$
with respect to $\vtheta$ and $\nabla_{\vtheta}h\left(\cdot,\vtheta\right)\in L^{2}(\Omega;\sR^M)$.  In the NTK regime, a DNN can be accurately approximated throughout the training as
\begin{equation*}
  h(\vx,\vtheta)\approx h\left(\vx,\vtheta_{0}\right)+\nabla_{\vtheta}h\left(\vx,\vtheta_{0}\right)\cdot\left(\vtheta-\vtheta_{0}\right),
\end{equation*}
which is the Taylor expansion of the DNN output function at initialization $\vtheta_{0}$ up to first order. Note that, for convenience, we also use $\vtheta_0$ for $\vtheta(0)$, i.e., the initial parameter set. 

We restrict our work in the the NTK regime of DNNs.
It has been shown in \citet{jacot2018neural,chizat2018note,lee2019wide} that, for
any $t\geq0$, by scaling the initial parameters of layer $l$ of width $m_l$ by $1/\sqrt{m_l}$ as $m_l \to \infty$ for all $l$, $\abs{\vtheta(t)-\vtheta(0)}\to 0$ indicating a NTK regime. Moreover, in \citet{arora2019exact}, a non-asymptotic proof is provided that relaxes the infinite width condition to certain sufficiently large width. For simplicity of our analysis, we consider the linearized DNN, i.e., 
\begin{equation}
  h(\vx,\vtheta)= h\left(\vx,\vtheta_{0}\right)+\nabla_{\vtheta}h\left(\vx,\vtheta_{0}\right)\cdot\left(\vtheta-\vtheta_{0}\right).\label{eq:linear}
\end{equation}
For the closeness of the linearied model and the DNN trained by gradient descent, We refer readers to \citet{arora2019exact}. In the NTK regime, for the loss function (also known as the emprical risk)
\begin{equation*}
  \RS(\vtheta)=\dist\left(\vh(\mX,\vtheta),\vY\right),
\end{equation*}
where $\dist$ is the distance function to be explained in Section \ref{sec:ROF}, 
the NTK gradient flow, i.e.,  gradient flow of the linearized
model $h(\vx,\vtheta(t))$, in parameter space follows
\begin{equation}
  \frac{\D \vtheta(t)}{\D t}=-\nabla_{\vtheta}\vh(\mX,\vtheta_{0})\nabla_{\vh(\mX,\vtheta(t))}\dist\left(\vh(\mX,\vtheta(t)),\vY\right),\label{eq:gdh}
\end{equation}
with initial value $\vtheta(0)=\vtheta_{0}$, where $\nabla_{\vtheta}\vh(\mX,\vtheta_{0})\in\sR^{m\times n}$, and 
$(\nabla_{\vtheta}\vh(\mX,\vtheta_{0}))_{ki}=\nabla_{\theta_{k}}h(\vx_{i},\vtheta_{0})$, and
$\nabla_{\vh(\mX,\vtheta(t))}\dist\left(\vh(\mX,\vtheta(t)),\vY\right)\in\sR^{n}$.  
In function space, the NTK gradient flow yields,
\begin{equation}
  \partial_{t}h(\vx,t)=-\vk(\vx,\mX)\nabla_{\vh(\mX,t)}\dist\left(\vh(\mX,t),\vY\right),\label{eq:gfh}
\end{equation}
with $h(\vx,t)=h(\vx,\vtheta(t))$, initial value $h(\cdot,0)=h(\cdot,\vtheta_{0})$, and kernel
$k\in L^{2}(\Omega\times\Omega)$ defined as
\begin{equation*}
  k(\cdot,\cdot)=\nabla_{\vtheta}h(\cdot,\vtheta_{0})^{\T}\nabla_{\vtheta}h(\cdot,\vtheta_{0}),
\end{equation*}
where $\nabla_{\vtheta}h(\cdot,\vtheta_{0})=(\partial_{\theta_{1}}h(\cdot,\vtheta_{0}),\cdots,\partial_{\theta_{m}}h(\cdot,\vtheta_{0}))^{\T}\in\sR^{m\times 1}$,
$\vk(\vx,\mX)\in\sR^{1\times n}$ for any $\vx\in\Omega$. Note that
Eq. \eqref{eq:gfh} of $h$ is a closed system. By \citet{jacot2018neural},
$k$ is symmetric and positive semi-definite. In the following, we
may denote $k_{\vtheta_0}(\cdot,\cdot)=\nabla_{\vtheta}h(\cdot,\vtheta_0)^{\T}\nabla_{\vtheta}h(\cdot,\vtheta_0)$
when we need to differentiate kernels corresponding to different architectures
or different initializations of DNNs.

\subsection{Reproducing kernel Hilbert space (RKHS)}

The kernel $k$ can induce a RKHS as follows. First, we cite the Mercer's
theorem (\citet{mercer1909xvi}).
\begin{thm}
  (Mercer's theorem (\citet{mercer1909xvi})) Suppose $k$ is a continuous
  symmetric positive semi-definite kernel. Then there is an orthonormal
  basis $\{\phi_{j}\}$ of $L^{2}(\Omega)$ consisting of eigenfunctions
  of $T_{k}$ defined as $\left[T_{k}g\right](\cdot)=\int_{\Omega}k(\cdot,\vx)g(\vx)\diff{\vx}$
  such that the corresponding sequence of eigenvalues $\sigma_{j}$
  is nonnegative. The eigenfunctions corresponding to non-zero eigenvalues
  are continuous on $\Omega$ and $k$ has the representation
  \begin{equation*}
    k(\vx,\vx')=\sum_{j=1}^{\infty}\sigma_{j}\phi_{j}(\vx)\phi_{j}(\vx'),
  \end{equation*}
  where the convergence is absolute and uniform.
\end{thm}

Then, we can define the RKHS as $H_{k}(\Omega):=\{g\in L^{2}(\Omega)\mid\sum_{i=1}^{\infty}\sigma_{i}^{-1}\left\langle g,\phi_{i}\right\rangle ^{2}_{L^2(\Omega)}<\infty\}$,
and the inner product in $H_{k}(\Omega)$ is given by
\begin{equation*}
  \left\langle f,g\right\rangle_{k}=\sum_{i=1}^{\infty}\sigma_{i}^{-1}\left\langle f,\phi_{i}\right\rangle _{L^{2}(\Omega)}\left\langle g,\phi_{i}\right\rangle _{L^{2}(\Omega)},
\end{equation*}
where $\left\langle g,\phi_{i}\right\rangle _{L^{2}(\Omega)}=\int_{\Omega}g(\vx)\phi_{i}(\vx)\diff{\vx}$.
Define $k^{-1}(\vx,\vx')=\sum_{i=1}^{\infty}\sigma_{i}^{-1}\phi_{i}(\vx)\phi_{i}(\vx')$,
then the kernel norm of any $g\in H_{k}(\Omega)$ can be
expressed as
\begin{equation*}
  \norm{g}_{k}=\left\langle g,g\right\rangle_{k}^{1/2}=\left(\int_{\Omega\times\Omega}g(\vx)g(\vx')k^{-1}(\vx,\vx')\diff{\vx'}\diff{\vx}\right)^{\frac{1}{2}}.
\end{equation*}
$H_{k}(\Omega)$ satisfies \citet{berlinet2004reproducing}:
(i) $\forall \vx\in\Omega,$$k(\cdot,\vx)\in H_{k}(\Omega)$;
(ii) $\forall \vx\in\Omega$, $\forall f\in H_{k}$, $\left\langle f(\cdot),k(\cdot,\vx)\right\rangle _{k}=f(\vx)$;
(iii) $\forall \vx,\vx'\in\Omega$, $\left\langle k(\cdot,\vx),k(\cdot,\vx')\right\rangle _{k}=k(\vx,\vx')$.

\section{Equivalent optimization problems in the NTK regime
\label{sec:ROF}}

For the completeness, we introduce the optimization problems equivalent to the training of DNNs in NTK regime, which appears explicitly or implicitly in previous works  \citep{chizat2018note,oymak2018overparameterized,jacot2018neural,mei2019mean,banburski2019theory}. As introduced in Section
\ref{subsec:Kernel-regime}, for the analysis of gradient flow of
$h_{\mathrm{DNN}}(\cdot,\vtheta_{\mathrm{DNN}}(t))$ in the kernel
regime, we focus on the gradient flow of its linearized model $h(\cdot,\vtheta(t))$,
i.e., Eqs. (\eqref{eq:gdh}, \eqref{eq:gfh}).

We consider the gradient flow under any loss $\RS(\vtheta)=\dist\left(\vh(\mX,\vtheta),\vY\right)$,
where $\dist$ is continuously differentiable and satisfies, for any $\vz\in\sR^{n}$,
(i) $\dist(\vz,\vz)=0$; (ii) $\dist(\vz',\vz)$ attains minimum if and only if $\vz'=\vz$.
(iii) $\vz'=\vz$ if and only if $\nabla_{\vz'}\dist(\vz',\vz)=0$. For example,
$\dist\left(\vh(\mX,\vtheta),\vY\right)=\frac{1}{n}\sum_{i=1}^{n}\left|h(\vx_{i},\vtheta)-y_{i}\right|^{p}$
for any $1<p<\infty$. By Theorem \ref{thm:equivalence} in Appendix
\ref{subsec:EquiTheorems}, the long time solution $\vtheta(\infty)=\lim_{t\to\infty}\vtheta(t)$
of dynamics \eqref{eq:gdh} is equivalent to the solution of the optimization
problem
\begin{equation}
  \min_{\vtheta}\norm{\vtheta-\vtheta_{0}}_{2},\quad s.t.,\quad \vh(\mX,\vtheta)=\vY.\label{eq:hopt}
\end{equation}
By Theorem \ref{thm:min kernel norm} in Appendix \ref{subsec:EquiTheorems},
$h(\vx,\vtheta(\infty))$ uniquely solves the optimization problem 
\begin{equation}
  \min_{h-h_\mathrm{ini}\in H_{k}(\Omega)}\norm{h-h_\mathrm{ini}}_{k},\quad s.t.,\quad \vh(\mX)=\vY,\label{eq: problem min kernel norm}
\end{equation}
where $h_\mathrm{ini}(\vx)=h(\vx,\vtheta_{0})$ and the constraints $\vh(\mX)=\vY$
are in the sense of trace (\citet{evans2010partial}, pp. 257--261).
The above results hold for any initial value $\vtheta_{0}$. We refer
to \emph{kernel-norm minimization framework} as using the optimization
problem (\ref{eq:hopt}) or (\ref{eq: problem min kernel norm}) to
analyze the long time solution of gradient flow dynamics in (\ref{eq:gdh})
or (\ref{eq:gfh}), respectively. 

With this framework, we emphasize the following results. First, for
a finite set of training data, given $\vtheta_{0}$, because $\dist$ is
absent in problems (\ref{eq:hopt}) and (\ref{eq: problem min kernel norm}),
the output function of a well-trained DNN in the NTK regime is
invariant to different choices of loss functions. Note that this result
is surprising in the sense that different $\dist$ clearly leads to different
trajectories of $\text{\ensuremath{\vtheta(t)}}$ and $h(\cdot,\vtheta(t))$.
Based on this result, it is not necessary to stick to commonly used
MSE loss for regression problems. For example, in practice, one can
use $\dist\left(\vh(\mX,\vtheta),\vY\right)=\frac{1}{n}\sum_{i=1}^{n}\left|h(\vx_{i},\vtheta)-y_{i}\right|^{p}$
of $1<p<2$ to accelerate the training of DNN near convergence or
$2<p<\infty$ to accelerate the training near initialization. One
can even mixing different loss functions to further boost the training
speed. Second, among all sets of parameters that fit the training
data, a DNN in the NTK regime always finds the one closest to the
initialization in the parameter space with respect to the $L^{2}$
norm. Third, in the functional space, this framework shows that DNNs
always seek to learn a function that has a shortest distance (with
respect to the kernel norm) to the initial output function. In the
following, we denote $h_{k}(\vx;h_\mathrm{ini},\mX,\vY)$ as the solution
of problem \eqref{eq: problem min kernel norm} depending on $k$,
$h_\mathrm{ini}$, $\mX$ and $\vY$. Remark that, for losses used in classification problems, e.g., cross-entropy, the global minimum is reached as $\vtheta \to \infty$, resulting a nonlinear training dynamics that cannot be captured by NTK. Therefore, our above results does not extend to these cases.

\section{Impact of non-zero initial output}

Problems (\ref{eq:hopt}) and (\ref{eq: problem min kernel norm})
explicitly incorporate the effect of initialization, thus enabling
us to study quantitatively its impact to the learning of DNNs. In
this section, we use the above framework to show that a random non-zero
initial DNN output leads to a specific type of generalization error.
We begin with a relation between the solution with zero initial output
and that with non-zero initial output. Proofs of the following theorems
can be found in Appendix \ref{sec:AppendixNegative-impact-of}.
\begin{thm}
  \label{thm:inipredict}For a fixed kernel function $k\in L^{2}(\Omega\times\Omega)$,
  and training set $\{\mX;\vY\}$, for any initial function $h_{\mathrm{ini}}\in L^{\infty}(\Omega)$,
  $h_{k}(\cdot;h_\mathrm{ini},\mX,\vY)$ can be decomposed as
  \begin{equation}
  h_{k}(\cdot;h_\mathrm{ini},\mX,\vY)=h_{k}(\cdot;0,\mX,\vY)+h_\mathrm{ini}-h_{k}(\cdot;0,\mX,h_\mathrm{ini}(\mX)).\label{eq:networkpredictini}
  \end{equation}
\end{thm}

This theorem unravels quantitatively the impact of a nonzero initialization,
i.e., $h_\mathrm{ini}\neq0$, to the output function of a well-trained
DNN in the NTK regime. Comparing the dynamics in \eqref{eq:gfh}
of zero and non-zero initialization, at the beginning, the difference
of DNN output is $h_{\mathrm{ini}}$, whereas, at the end of the training,
that difference shrinks to $h_\mathrm{ini}-h_{k}(\cdot;0,\mX,h_\mathrm{ini}(\mX))$,
which is the residual of fitting $h_{\mathrm{ini}}$ sampled at $\mX$
by the same DNN. Note that \citet{geiger2019scaling}   figures out qualitatively
that $h_{\mathrm{ini}}$, which does not vanish as the width of DNN
tends to infinity, decreases during the training. However, they do
not arrive at a quantitative relation as revealed by Theorem \ref{thm:inipredict}.

The expected generalization error of DNN with a random non-zero initial
output can be estimated as follows.
\begin{thm}
  \label{thm:iniextraerror}For a target function $f\in L^{\infty}(\Omega)$,
  if $h_{\mathrm{ini}}$ is generated from an unbiased distribution of random functions
   $\mu$ such that $\Exp_{h_\mathrm{ini}\sim \mu}h_{\mathrm{ini}}=0$,
  then the generalization error of $h_{k}(\cdot;h_\mathrm{ini},\mX,f(\mX))$
  can be decomposed as follows 
  \begin{align*}
    \Exp_{h_\mathrm{ini}\sim \mu}\RS\left(h_{k}(\cdot;h_\mathrm{ini},\mX,f(\mX)),f\right)
    &= \RS\left(h_{k}(\cdot;0,\mX,f(\mX)),f\right)\\
    &~~~~+\Exp_{h_\mathrm{ini}\sim \mu}\RS\left(h_{k}(\cdot;0,\mX,h_\mathrm{ini}(\mX)),h_\mathrm{ini}\right),
  \end{align*}
  where $\RS(h_{k}(\cdot;h_\mathrm{ini},\mX,f(\mX)),f)=\norm{h_{k}(\cdot;h_\mathrm{ini},\mX,f(\mX))-f}_{L^{2}(\Omega)}^{2}$.
\end{thm}

By the above theorem, $\Exp_{h_\mathrm{ini}\sim \mu}\RS\left(h_{k}(\cdot;0,\mX,h_\mathrm{ini}(\mX)),h_\mathrm{ini}\right)\geq0$
is a specific type of generalization error induced by $h_{\mathrm{ini}}$.
Clearly, this error decreases as the sample size $n$ increases and
as $n\to\infty$, $h_\mathrm{ini}-h_{k}(\cdot;0,\mX,h_\mathrm{ini}(\mX))\to0$,
which conforms with our intuition that if the optimization is sufficiently
constrained by the training data, then the effect of initialization
can be ignored. For real datasets of a limited number of training
samples, this error is in general non-zero. By F-Principle \citep{xu_training_2018,xu2019frequency},
DNNs tend to fit training data by low frequency functions. Therefore,
qualitatively, $h_\mathrm{ini}-h_{k}(\cdot;0,\mX,h_\mathrm{ini}(\mX))$
consists mainly of the high frequencies of $h_\mathrm{ini}$ which
cannot be well constrained at $\mX$.

\section{AntiSymmetrical Initialization trick (ASI)}

In general, from the Bayesian inference perspective, for fixed $k$,
a random $h_{\mathrm{ini}}$ introduces a prior to the inference that
is irrelevant to the target function, thus should lower the accuracy
of inference. Specifically, the scaling of $1/\sqrt{m_l}$ for NTK leads to a random initial function $h_{\rm{ini}}\sim O(1)$, which cannot be neglected for the analysis. To eliminate the negative impact of non-zero initial
DNN output, a naive way is to set the initial parameters sufficiently small. However, because $\vtheta=\vzero$ is a high order saddle point of $\RS(\vtheta)$, too small initial parameters can lead to a nonlinear training dynamics that cannot be captured by NTK. In practice, small initial parameters can also results in a different initial kernel of DNN with a vanishing magnitude, which slows down the training and changes the learning result. Based on our above theoretical
results, we design an AntiSymmetrical Initialization trick (ASI) which
can fix the initial output to zero but also keep the kernel invariant.
Let $h_{i}^{[l]}$ be the output of the $i$-th node of the $l$-th
layer of a $L$ layer DNN. Then, $h_{i}^{[l]}(\vx)=\sigma_{i}^{[l]}(\mW_{i}^{[l]}\cdot \vh^{[l-1]}(\vx)+b_{i}^{[l]})$,
for $i=1,\cdots,m_{l}$. For the $i$-th neuron of the output layer
of DNN, $h_{i}^{[L]}(\vx)=\mW_{i}^{[L]}\cdot \vh^{[L-1]}+b_{i}^{[L]}.$
After initializing the DNN by any method, we obtain $h^{[L]}(\vx,\vtheta(0))$,
where 
\begin{equation*}
  \vtheta(0)=(\mW^{[L]}(0),\vb^{[L]}(0),\mW^{[L-1]}(0),\vb^{[L-1]}(0),\cdots,\vb^{[1]}(0)).
\end{equation*}
The ASI for general loss functions is to consider a new DNN expressed
as $h_{\mathrm{ASI}}(\vx,\vTheta(t))=\frac{\sqrt{2}}{2}h^{[L]}(\vx,\vtheta(t))-\frac{\sqrt{2}}{2}h^{[L]}(\vx,\vtheta'(t))$
where $\vTheta=(\vtheta,\vtheta')$, $\vTheta$ is initialized such that
$\vtheta'(0)=\vtheta(0)$. In the following, we prove a theorem that
ASI trick eliminates the nonzero random prior without changing the
kernel $k$ (Proof can be found in Appendix \ref{sec:AppendixASI}).
\begin{thm}
\label{thm:kernelinv} For any general loss function $\dist$ satisfying
the conditions in Sec. \ref{sec:ROF}, in the NTK regime, the gradient
flow of both $h(\vx,\vtheta(t))$ and $h_{\mathrm{ASI}}(\vx,\vTheta(t))$
follows the kernel dynamics
\begin{equation}
\partial_{t}h'=-\vk(\cdot,\mX)\nabla_{\vh(\mX,t)}\dist\left(\vh'(\mX,t),\vY\right),\label{eq:samedy}
\end{equation}
with initial value $h'(\cdot,0)=h_{\mathrm{ini}}=h(\vx,\vtheta(0))$
and $h'(\cdot,0)=0$, respectively, where $\{\mX;\vY\}$ is the training
set, $k(\vx,\vx')=k_{\vtheta_0}(\vx,\vx')=\nabla_{\vtheta}h(\vx,\vtheta_{0})^{\T}\nabla_{\vtheta}h(\vx',\vtheta_{0})$.
\end{thm}

Note that \citet{chizat2018note} proposes a ``doubling trick''
to offset the initial DNN output, that is, neurons in the last layer
are duplicated, with the new neurons having the same input weights
but opposite output weights. By applying the ``doubling trick'',
$h'(\cdot,0)=0$. However, the kernel of layers $L-1$ and $L$ doubles,
whereas the kernel of layers $m\leq L-2$ completely vanishes
(See Appendix \ref{sec:doubling-trickappendix} for the proof), which
could have large impact on the training dynamics as well as the generalization
performance of DNNs.

\section{Experiments}

Our above theoretical results are obtained using the linearized model
of DNN in Eq. \eqref{eq:linear} that well approximates the behavior
of DNN in the NTK regime. In this section, we will demonstrate
experimentally the accuracy of these results for very wide DNNs and
the effectiveness of these results for general DNNs. First, using
synthetic data, we verify the invariance of DNN output after training
to different loss functions as studied in Sec. \ref{sec:ROF}. Then,
we verify the linear relation in Eq. (\ref{eq:networkpredictini}).
Moreover, we demonstrate the effectiveness of the ASI trick on both
synthetic data and the MNIST dataset. Here is a summary of the settings
of DNNs in our experiments. The activation function is ReLU, parameters
are initialized by a Gaussian distribution with mean $0$ and standard
deviation $v_{{\rm std}}\sqrt{2/(m_{\text{in}}+m_{\text{out}})}$,
where $m_{{\rm in}}$ and $m_{{\rm out}}$ are for the input and the
output dimension of the neuron, respectively. For Figs. (\ref{fig:sin-relu-6-2},
\ref{fig:sin-relu-6}, \ref{fig:Boston}), networks are trained by
full gradient descent with MSE loss and the learning rate is $10^{-5}$.

\subsection{Invariance of DNN output to loss functions}

For a DNN $h(\vx,\vtheta(t))$ with initialization fixed at certain $\vtheta(0)=\vtheta_{0}$,
we consider its gradient descent training for two loss functions:
the $L^{2}$ (MSE) loss $\dist(\vh(\mX,\vtheta),\vY)=\frac{1}{n}\sum_{i=1}^{n}(h(\vx_{i},\vtheta)-y_{i})^{2}$
and the $L^{4}$ loss $\dist(\vh(\mX,\vtheta),\vY)=\frac{1}{n}\sum_{i=1}^{n}(h(\vx_{i},\vtheta)-y_{i})^{4}$.
In Fig. \ref{fig:sin-relu-6-2}, as Theorems \ref{thm:equivalence}
and \ref{thm:min kernel norm} predict, the well-trained DNN outputs
for these two losses overlap very well not only at $4$ training points,
but also at all the test points.
\begin{center}
\begin{figure}
\begin{centering}
\includegraphics[scale=0.6]{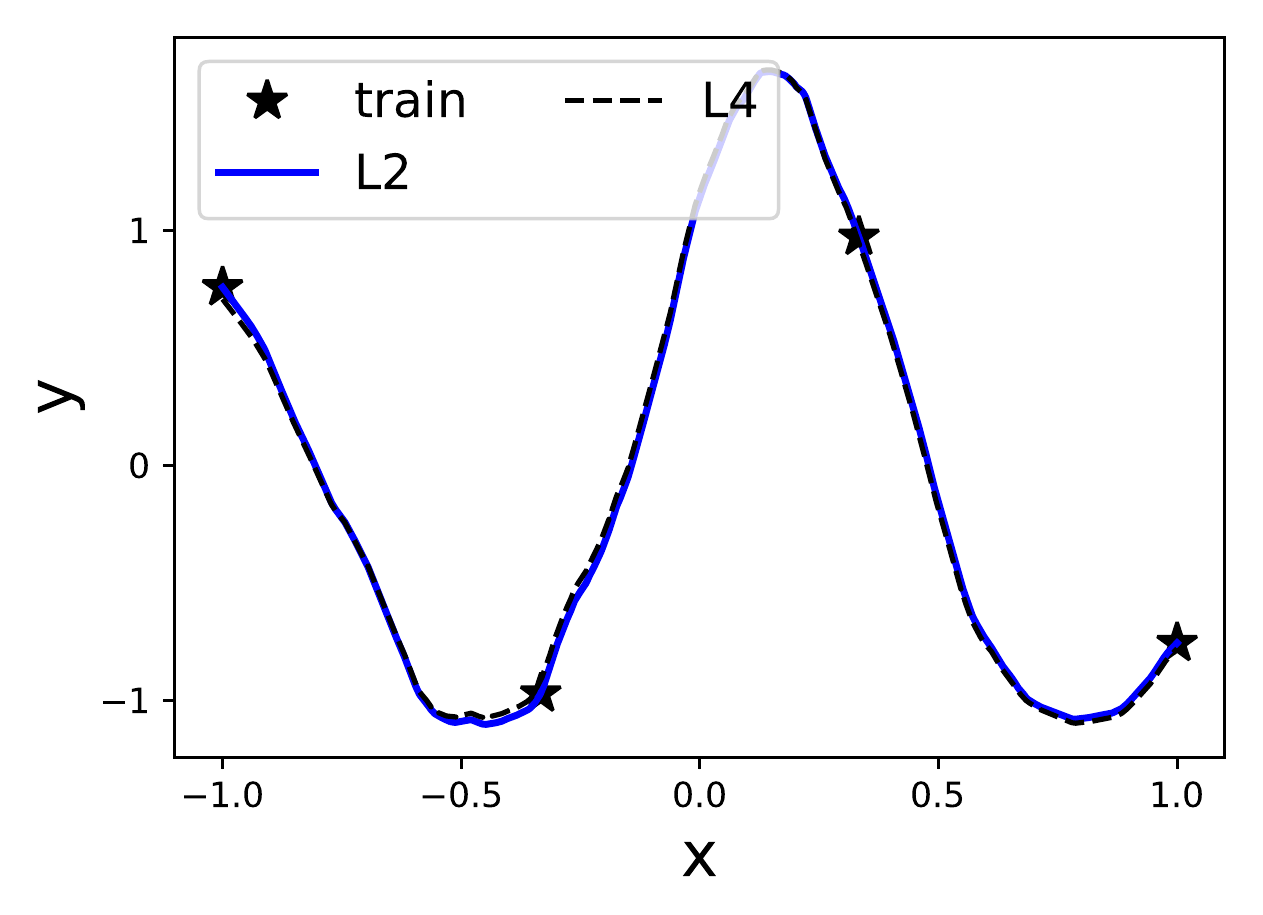} 
\par\end{centering}
\caption{Invariance of DNN output to loss functions. Black stars indicate training
data. Blue solid curve and the black curve indicate the outputs on
the test samples of the DNN well-trained by $L^{2}$ loss and $L^{4}$
loss, respectively. The size of DNN is 1-500-500-1. $v_{{\rm std}}=5$.
The training and test data are randomly sampled from $\sin(4x)$ in
$[-1,1]$ with sample size $4$ and $500$, respectively. \label{fig:sin-relu-6-2} }
\end{figure}
\par\end{center}

\subsection{Linear relation and the effectiveness of ASI trick}

\subsubsection{1-d synthetic data}

In this sub-section, we use 1-d data, which is convenient for visualization,
to train DNNs of a large width. As shown in Fig. \ref{fig:sin-relu-6}(a),
without applying any trick, the original DNN initialized with a large
weight learns/interpolates the training data in a fluctuating manner
(blue solid). Both the ASI trick (cyan dashed dot) and the ``doubling
trick'' (green dashed) enable the DNN to interpolate the training
data in a more ``flat'' way. As shown by the red dashed curve, the
output computed by the right hand side (RHS) of Eq. (\ref{eq:networkpredictini})
accurately predicts the final output of the original DNN on test points.
In our experiments $h_{k}(x;0,\mX,h_\mathrm{ini}(\mX))$ , $h_{k}(x;0,\mX,\vY)$,
$h_{k}(x;h_{\mathrm{ini}},\mX,\vY)$ are always obtained using very wide
DNNs with or without the ASI trick applied. From Eq. (\ref{eq:networkpredictini}),
a non-zero initialization adds a prior $h_\mathrm{ini}-h_{k}(x;0,\mX,h_\mathrm{ini}(\mX))$
to the final DNN output. As shown by the cyan dashed curve in Fig.
\ref{fig:sin-relu-6}(b), this prior fluctuates a lot, thus, leading
to an oscillatory output of DNN after training. Note that this experiment
also support the prediction of F-Principle \citep{xu_training_2018,xu2019frequency}
that it is the high frequencies of $h_{\mathrm{ini}}$ (red dashed)
that remains in the final output of DNN. Concerning the training speed,
as shown in Fig. \ref{fig:sin-relu-6}(c), the loss function of the
DNN with the ASI trick applied decreases much faster than that of
the original DNN or the one with the ``doubling trick'' applied.
For a reference, we double the original network similar to the ASI
trick, then initialize it randomly following the same distribution
as the original. We refer to this trick as RND. As shown by the black
curve in Fig. \ref{fig:sin-relu-6}(c), the loss function of the DNN
with the RND trick applied also decreases much slower than the one
with the ASI trick applied.

In summary, for 1-d problem, the linear relation holds well and
the ASI trick is effective in removing the artificial prior induced
by $h_{\mathrm{ini}}$ and accelerating the training speed. In the
following, we further investigate the generalization performance of
DNN with the ASI trick  for real datasets.


\begin{center}
\begin{figure}
\begin{centering}
\includegraphics[scale=0.78]{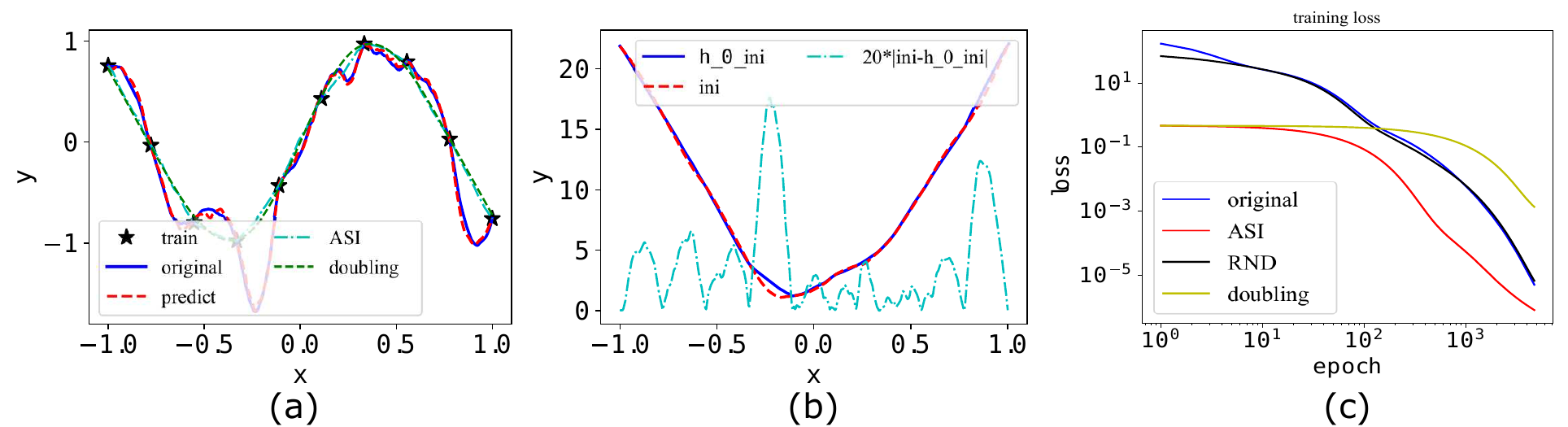} 
\par\end{centering}
\caption{Synthetic data. (a) Black stars indicate the training data. Other
curves indicate final outputs of different DNNs evaluated at the test
points. Blue solid: the original DNN (without tricks); cyan dashed
dot: DNN with ASI trick applied; green dashed: DNN with ``doubling
trick'' applied; red dashed: the RHS of Eq. (\ref{eq:networkpredictini}).
(b) Blue: $h_{k}(x;0,\mX,h_\mathrm{ini}(\mX))$; red: $h_\mathrm{ini}$;
cyan: $20|h_\mathrm{ini}-h_{k}(x;0,\mX,h_\mathrm{ini}(\mX))|$. (c) Evolution
of loss functions of different DNNs during the training. Blue: the
original DNN; red: DNN with ASI trick applied; black: DNN with RND
applied; yellow: DNN with the ``doubling trick'' applied. The width
of the original DNN is 1-5000-5000-1. $v_{{\rm std}}=10$. We sample training
and test data   randomly   from $\sin(4x)$ in $[-1,1]$ with
size $10$ and $500$, respectively. \label{fig:sin-relu-6} }
\end{figure}
\par\end{center}

\subsubsection{Boston house price dataset}

We verify our theoretical results for high dimensional regression
problems using Boston house price dataset \citep{harrison1978hedonic},
in which we normalize the value of each property and the price to
$[-0.5,0.5]$. We choose $400$ samples as the training data, and
the other $106$ samples as the test data. As illustrated by the red
dots concentrating near the black line of an identity relation in
Fig. \ref{fig:Boston}a, the RHS of Eq. (\ref{eq:networkpredictini})
well predicts the final output of the original DNN without any trick,
which is significant different from the final output of DNN with the
ASI trick applied as shown by the blue dots deviating from the black
line. As shown in \ref{fig:Boston}b, similar to the experiments on
1-d synthetic data, the ASI trick accelerates the training. In addition,
conforming with Theorem \ref{thm:iniextraerror}, the generalization
error of the DNN with ASI trick applied is much smaller than that
of the original DNN. 
\begin{center}
\begin{figure}
\begin{centering}
\includegraphics[scale=0.88]{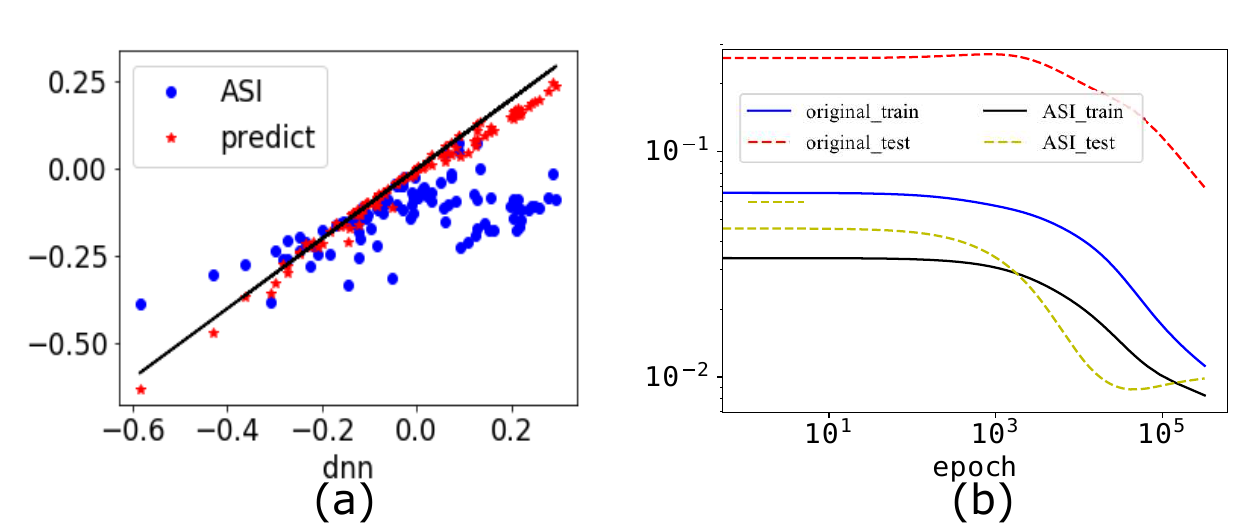} 
\par\end{centering}
\caption{Boston house price dataset. (a) Each dot represents outputs evaluated
at one test point. The abscissa is $h_{k}(\cdot;h_\mathrm{ini},\mX,\vY)$
obtained using the original DNN. The ordinate for each blue dot is
$h_{k}(x;0,\mX,\vY)$ obtained using DNN with the ASI trick applied, whereas
for each red dot is the RHS of Eq. (\ref{eq:networkpredictini}).
The black line indicates the identity function $y=x$. (b) The evolution
of training loss (blue solid) and test loss (red dashed) of the original
DNN, and the training loss (black solid) and test loss (yellow dashed)
of DNN with ASI trick applied. The width of DNN is 13-100000-1. $v_{{\rm std}}=5$.
\label{fig:Boston} }
\end{figure}
\par\end{center}

\subsubsection{MNIST dataset and the non-NTK regime of DNN}

Next, we use the MNIST dataset to examine the effectiveness of ASI
trick in the non-NTK regime of DNNs. We use a DNN with a more realistic
setting of width 784-400-400-400-400-10, cross-entropy loss, batch
size 512, and Adam optimizer \citep{kingma2014adam}. In such a case,
as shown in Fig. \ref{fig:mce}, the ASI trick still effectively eliminate
$h_{\mathrm{ini}}$, accelerate the training speed and improve the
generalization. In Fig. \ref{fig:mce}(b), with the ASI trick applied,
both training and test accuracy exceeds $90\%$ after only $1$ epoch
of training. This phenomenon further demonstrate that, without the
interference of $h_{\mathrm{ini}}$, DNNs can capture very efficiently
and accurately the behavior of the training data. 
\begin{center}
\begin{figure}
\begin{centering}
\includegraphics[scale=0.88]{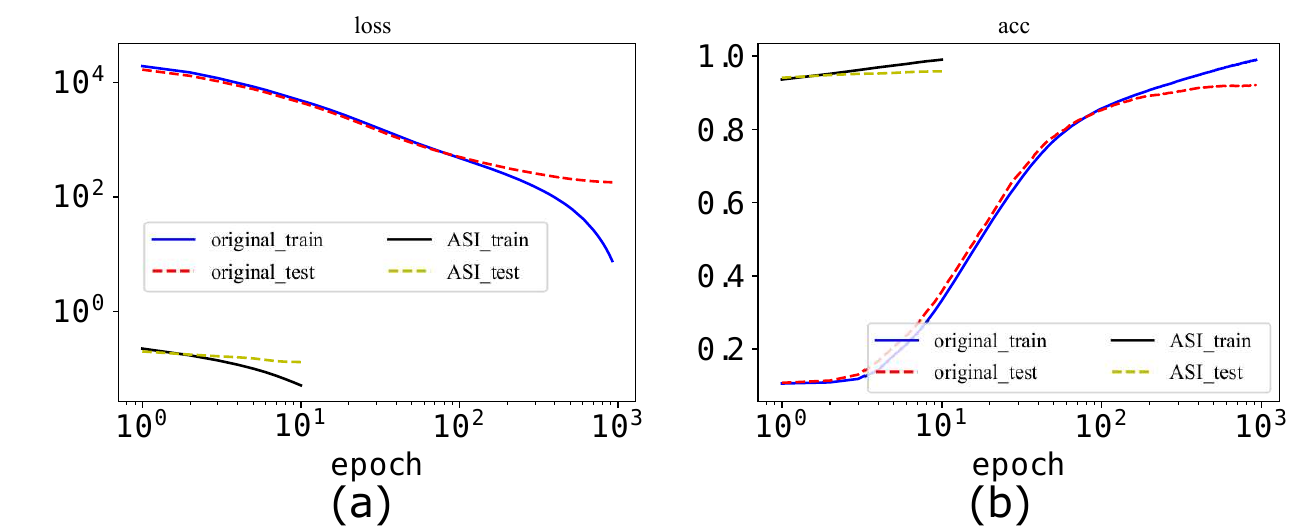} 
\par\end{centering}
\caption{Effectiveness of ASI trick for MNIST dataset in the non-NTK regime
of DNN. (a) Evolution of loss functions with the same legend as in
Fig. \ref{fig:Boston}(b). (b) Evolution of the corresponding accuracy.
The learning rate is $2\times10^{-7}$. See main text for other settings.\label{fig:mce} }
\end{figure}
\par\end{center}

\section{Discussion}

In this work, focusing on the regression problem, we propose using a kernel-norm
minimization framework to study theoretically the role of loss function
and initialization for DNNs in the NTK regime. We prove that, given
initialization, DNNs of different loss functions in a general class
find the same global minimum. Regarding initialization, we find that
a non-zero initial output of DNN leads to a specific type of generalization
error. We then propose the ASI trick to eliminate this error without
changing the neural tangent kernel. Experimentally, we find that ASI
trick significantly accelerates the training and improves the generalization
performance. Moreover, ASI trick remains effective for classification
problems as well as for DNNs in the non-NTK regime. Because the
error of DNN output induced by random initialization shrinks during
the training, the advantage of ASI trick is much more significant
at the early stage of the training. Based on above results, we suggest
incorporating ASI trick in the design of controlled experiments for
the quantitative study of DNNs. 

From the perspective of training flexibility, ASI trick can alleviate
the sensitivity of generalization and training speed to different
random initializations of DNNs, thus expand the range of well-generalized
initializations. This property could be especially helpful for finding
a well-generalized solution of a new problem when empirical guidance
is not available. We also remarks that, from Eq. (\ref{eq:networkpredictini}),
a particular prior of $h_\mathrm{ini}$, such as the one from meta
learning \citep{rabinowitz2019meta}, could decrease the generalization
error. However, when meta learning is not available, a zero $h_\mathrm{ini}$
is in general the best choice for generalization.

Cross-entropy loss is commonly used in classification problems, for
which the DNN outputs are often transformed by a softmax function
to stay in $(0,1)$. Theoretically, to obtain a zero cross-entropy
loss given that labels of the training data take $1$ or $0$, weights
of the DNN should approach infinity. In such a case, it is impossible
for a DNN to stay in the NTK regime, which requires a small variation
of weights throughout the training. However, in practice, training
of a DNN often stops by meeting certain criteria of training accuracy
or validation accuracy. Therefore, it is possible that weights of
a sufficiently wide DNN stay in a small neighborhood of the initialization
during the training. By setting a proper tolerance for the cross-entropy
loss, we will analyze in the future the behavior of DNNs in kernel
regime for classification problems with cross-entropy loss.

 \acks{Zhiqin Xu is supported by the Student Innovation Center at Shanghai Jiao Tong University. }

\bibliography{DLRef} 

\part*{Appendix}

\section{Theorems for the kernel-norm minimization framework \label{subsec:EquiTheorems}}
\begin{thm}
  \label{thm:equivalence}Let $\vtheta(t)$ be the solution of gradient
  flow dynamics
  \begin{equation}
    \frac{\D }{\D t}\vtheta(t)=-\nabla_{\vtheta}\vh(\mX,\vtheta_{0})\nabla_{\vh(\mX,\vtheta(t))}\dist\left(\vh(\mX,\vtheta(t)),\vY\right)\label{eq:gf-thm2}
  \end{equation}
  with initial value $\vtheta(0)=\vtheta_{0}$, where $\nabla_{\vtheta}\vh(\mX,\vtheta_{0})$
  is a full rank (rank $n$) matrix of size $m\times n$
  with $m>n$. Then $\vtheta(\infty)=\lim_{t\to\infty}\vtheta(t)$
  exists and uniquely solves the constrained optimization problem 
  \begin{equation}
    \min_{\vtheta}\norm{\vtheta-\vtheta_{0}}_{2},\ s.t.,\ \vh(\mX,\vtheta)=\vY.\label{eq:min norm}
  \end{equation}
\end{thm}

\begin{rmk*}
Compared with the nonlinear gradient flow of DNN, the linearization
in Eq. \eqref{eq:gf-thm2} is only performed on the hypothesis function
$h$ but not on the loss function or the gradient flow.
\end{rmk*}
\begin{proof}
  Gradient flow Eq. \eqref{eq:gf-thm2} can be written as 
  \begin{equation*}
    \frac{\D\vtheta(t)}{\D t}=-\nabla_{\vtheta}\dist\left(\vh(\mX,\vtheta(t)),\vY\right).
  \end{equation*}
  Then denote $\RS(\vtheta(t))=\dist\left(\vh(\mX,\vtheta(t)),\vY\right)$,
  \begin{equation*}
    \Abs{\frac{\D \vtheta}{\D t}}^{2}=-\frac{\D }{\D t}\RS(\vtheta(t)).
  \end{equation*}
  Note that $\RS(\vtheta(t))=\dist\left(\vh(\mX,\vtheta(t)),\vY\right)\geq0$ for any $t\geq0$. Thus
  \begin{align*}
    \int_{0}^{\infty}\Abs{\frac{\D \vtheta}{\D t}}^{2}\diff{t}& =\RS(\vtheta(0))-\RS(\vtheta(\infty)) \leq \RS(\vtheta(0)).
  \end{align*}
  Since $\frac{\D \vtheta}{\D t}$ is continuous, 
  \begin{equation*}
    \lim_{t\to\infty}\frac{\D }{\D t}\vtheta(t)=\lim_{t\to\infty}-\nabla_{\vtheta}\vh(\mX,\vtheta_{0})\nabla_{\vh(\mX,\vtheta(t))}\dist\left(\vh(\mX,\vtheta(t)),\vY\right)=\vec{0}.
  \end{equation*}
  Because $\nabla_{\vtheta}\vh(\mX,\vtheta_{0})$ is a full rank matrix,
  \begin{equation*}
    \lim_{t\to\infty}\nabla_{\vh(\mX,\vtheta(t))}\dist\left(\vh(\mX,\vtheta(t)),\vY\right)=\vec{0}.
  \end{equation*}
  Recall that $\nabla_{\vz'}\dist(\vz',\vz)=\vec{0}$ if and only if $\vz'=\vz$. Thus
  \begin{equation*}
    \lim_{t\to\infty}\vh(\mX,\vtheta(t))=\vY.
  \end{equation*}
  By applying singular value decomposition to $\nabla_{\vtheta}\vh(\mX,\vtheta_{0})$,
  we obtain $\nabla_{\vtheta}\vh(\mX,\vtheta_{0})=\mV\mSigma \mU^{\T}$, where
  $\mV$ and $\mU$ are orthonormal matrix of size $m\times m$
  and $n\times n$ respectively, $\mSigma=\left[\begin{array}{c}
  \mSigma_{1}\\
  \mzero
  \end{array}\right]$ of size $m\times n$, where $\mSigma_{1}$ is a full
  rank diagonal matrix of size $n\times n$. $\mV$ can be split into
  two part as $\mV=[\mV_{1},\mV_{2}]$, where $\mV_{1}$ takes the first $n$
  columns and $\mV_{2}$ takes the last $m-n$ columns of
  $\mV$. Then 
  \begin{align*}
    \nabla_{\vtheta}\vh(\mX,\vtheta_{0})=\mV\mSigma \mU^{\T}
     & =[\mV_{1},\mV_{2}]\left[\begin{array}{c}
    \mSigma_{1}\\
    \mzero
    \end{array}\right]\mU^{\T} =\mV_{1}\mSigma_{1}\mU^{\T},\\
    \mV_{2}^{\T}\nabla_{\vtheta}\vh(\mX,\vtheta_{0}) & =\mV_{2}^{\T}\mV_{1}\mSigma_{1}\mU^{\T} =\mzero.
  \end{align*}
   Therefore 
  \begin{equation*}
    \frac{\D }{\D t}\mV_{2}^{\T}\vtheta(t)=-\mV_{2}^{\T}\nabla_{\vtheta}\vh(\mX,\vtheta_{0})^{\T}\nabla_{\vh(\mX,\vtheta(t))}\dist\left(\vh(\mX,\vtheta(t)),\vY\right)=\mzero,
  \end{equation*}
  which leads to 
  \begin{equation}
    \mV_{2}^{\T}\left(\vtheta(t)-\vtheta_{0}\right)=\mzero,\ {\rm for\ any}\ t\geq0.\label{eq:v2}
  \end{equation}
  By Eq. \eqref{eq:linear}, $\lim_{t\to\infty}\vh(\mX,\vtheta(t))=\vY$ yields
  \begin{equation*}
    \lim_{t\to\infty}\nabla_{\vtheta}\vh(\mX,\vtheta_{0})^{\T}\left(\vtheta(t)-\vtheta_{0}\right)=\vY-\vh(\mX,\vtheta_{0}),
  \end{equation*}
  which can be written as
  \begin{equation*}
    \lim_{t\to\infty}\mU\mSigma_{1}\mV_{1}^{\T}\left(\vtheta(t)-\vtheta_{0}\right)=\vY-\vh(\mX,\vtheta_{0})
  \end{equation*}
  hence
  \begin{equation}
    \lim_{t\to\infty}\mV_{1}^{\T}\left(\vtheta(t)-\vtheta_{0}\right)=\mSigma_{1}^{-1}\mU^{\T}\left[\vY-\vh(\mX,\vtheta_{0})\right].\label{eq:v1}
  \end{equation}
  Combining Eq. \eqref{eq:v2} and \eqref{eq:v1}, $\vtheta(\infty)=\lim_{t\to\infty}\vtheta(t)$
  exists and is uniquely determined as
  \begin{align*}
    \mV^{\T}\left(\vtheta(\infty)-\vtheta_{0}\right) & =\left[\begin{array}{c}
    \mV_{1}^{\T}\\
    \mV_{2}^{\T}
    \end{array}\right]\left(\vtheta(\infty)-\vtheta_{0}\right)\\
     & =\left[\begin{array}{c}
    \mSigma_{1}^{-1}\mU^{\T}\left[\vY-\vh(\mX,\vtheta_{0})\right]\\
    \mzero
    \end{array}\right],
  \end{align*}
  \begin{equation*}
    \vtheta(\infty)-\vtheta_{0}=\mV\left[\begin{array}{c}
    \mSigma_{1}^{-1}\mU^{\T}\left[\vY-\vh(\mX,\vtheta_{0})\right]\\
    \mzero
    \end{array}\right]=\mV_{1}\mSigma_{1}^{-1}\mU^{\T}\left[\vY-\vh(\mX,\vtheta_{0})\right],
  \end{equation*}
  which leads to 
  \begin{equation*}
    \vtheta(\infty)=\mV_{1}\mSigma_{1}^{-1}\mU^{\T}\left[\vY-\vh(\mX,\vtheta_{0})\right]+\vtheta_{0}.
  \end{equation*}
  On the other hand, by the above analysis, problem \eqref{eq:min norm}
  can be formulated as 
  \begin{equation*}
    \min_{\vtheta}\norm{\vtheta-\vtheta_{0}}_{2}\quad \mathrm{s.t.}\quad \mV_{1}^{\T}\left(\vtheta-\vtheta_{0}\right)=\mSigma_{1}^{-1}\mU^{\T}\left[\vY-\vh(\mX,\vtheta_{0})\right].
  \end{equation*}
  Any $\vtheta$ satisfies above constraint can be expressed as 
  \begin{equation*}
    \vtheta=\mV_{1}\mSigma_{1}^{-1}\mU^{\T}\left[\vY-\vh(\mX,\vtheta_{0})\right]+\mV_{2}\vxi+\vtheta_{0},
  \end{equation*}
  where $\vxi\in\sR^{m-n}$. Then 
  \begin{equation*}
    \norm{\vtheta-\vtheta_{0}}_{2}^{2}=\norm{\mV_{1}\mSigma_{1}^{-1}\mU^{\T}\left[\vY-\vh(\mX,\vtheta_{0})\right]}_{2}^{2}+\norm{\mV_{2}\vxi}_{2}^{2}.
  \end{equation*}
  Clearly, $\norm{\vtheta-\vtheta_{0}}_{2}$ attains minimum
  if and only if $\vxi=\mzero$. Therefore $\vtheta(\infty)=\mV_{1}\mSigma_{1}^{-1}\mU^{\T}\left[\vY-\vh(\mX,\vtheta_{0})\right]+\vtheta_{0}$
  uniquely solves problem \eqref{eq:min norm}.
\end{proof}
For the proof of Theorem \ref{thm:min kernel norm}, we first introduce
the following two lemmas.
\begin{lem}
  \label{lem:htotheta}
  For any $h'\in H_{k}(\Omega)$, there
  exist $\vtheta'=\left\langle h'(\cdot),\nabla_{\vtheta}h(\cdot,\vtheta_{0})\right\rangle _{k}$
  such that $h'=\nabla_{\vtheta}h(\vx,\vtheta_{0})^{\T}\vtheta'$.
\end{lem}

\begin{proof}
  For any $h'\in H_{k}(\Omega)$, 
  \begin{align*}
    \left\langle h'(\cdot),k(\cdot,\vz)\right\rangle _{k} 
    &= \left\langle h'(\cdot),\nabla_{\vtheta}h(\cdot,\vtheta_{0})^{\T}\nabla_{\vtheta}h(\vz,\vtheta_{0})\right\rangle _{k}\\
    &= \left\langle h'(\cdot),\nabla_{\vtheta}h(\cdot,\vtheta_{0})\right\rangle _{k}^{\T}\nabla_{\vtheta}h(\vz,\vtheta_{0}).
  \end{align*}
  For $\vtheta'=\left\langle h'(\cdot),\nabla_{\vtheta}h(\cdot,\vtheta_{0})\right\rangle _{k}$,
  by the property of reproducing kernel $k$,
  \begin{equation*}
    h'(\vx)=\left\langle h'(\cdot),k(\cdot,\vx)\right\rangle_{k}=\nabla_{\vtheta}h(\vx,\vtheta_{0})^{\T}\vtheta'.
  \end{equation*}
\end{proof}
\begin{lem}
    \label{lem:thetatoh} For any $\vtheta'\in\sR^{m}$,
    $\nabla_{\vtheta}h(\cdot,\vtheta_{0})^{\T}\vtheta'\in H_{k}(\Omega)$. 
\end{lem}

\begin{proof}
  By the Mercer's theorem, 
  \begin{equation*}
    k(\vx,\vx')=\sum_{j=1}^{\infty}\sigma_{j}\phi_{j}(\vx)\phi_{j}(\vx'),
  \end{equation*}
  where $\{\phi_{j}\}_{j=1}^{\infty}$ are orthonormal basis of $L^{2}(\Omega)$.
  Suppose that $\nabla_{\vtheta}h(\cdot,\vtheta_{0})^{\T}\vtheta'\notin H_{k}(\Omega)$.
  Let $\sigma_\mathrm{min}=\min_{j}\{\sigma_j\}$. If $\sigma_\mathrm{min}>0$, then 
  \begin{align*}
      \sum_{i=1}^\infty \sigma_i^{-1}\langle \nabla_{\vtheta}h(\cdot,\vtheta_{0})^{\T}\vtheta',\phi_i\rangle^2_{L^2(\Omega)}
      & \leq
      \sigma_\mathrm{\min}^{-1}\sum_{i=1}^\infty \langle \nabla_{\vtheta}h(\cdot,\vtheta_{0})^{\T}\vtheta',\phi_i\rangle^2_{L^2(\Omega)}\\
      & \leq
      \sigma_\mathrm{\min}^{-1}\norm{\nabla_{\vtheta}h(\cdot,\vtheta_{0})^{\T}\vtheta'}^2_{L^2(\Omega)}
      <\infty.
  \end{align*}
  If $\sigma_\mathrm{min}=0$,
  then there exists $j_{1}$ such that $\sigma_{j_{1}}=0$ and $\left\langle \nabla_{\vtheta}h(\cdot,\vtheta_{0})^{\T}\vtheta',\phi_{j_{1}}\right\rangle _{L^{2}(\Omega)}\neq0$.
  Then there exists $j_{2}$ such that $\left\langle \nabla_{\theta_{j_2}}h(\cdot,\vtheta_{0}),\phi_{j_{1}}\right\rangle _{L^{2}(\Omega)}\neq0$,
  where $\theta_{j_2}$ is the $j_2$-th component of $\vtheta$.
  Then 
  \begin{align*}
    \int_{\Omega}\phi_{j_{1}}(\vx)k(\vx,\vx')\phi_{j_{1}}(\vx')\diff{\vx}\diff{\vx'} & =\int_{\Omega}\phi_{j_{1}}(\vx)\left(\nabla_{\vtheta}h(\vx,\vtheta_{0})^{\T}\nabla_{\vtheta}h(\vx',\vtheta_{0})\right)\phi_{j_{1}}(\vx')\diff{\vx}\diff{\vx'}\\
    &= \sum_{i}\left\langle \partial_{\theta_{i}}h(\cdot,\vtheta_{0}),\phi_{j_{1}}\right\rangle _{L^{2}(\Omega)}^{2}\\
    &\geq \left\langle \partial_{\theta_{j_2}}h(\cdot,\vtheta_{0}),\phi_{j_{1}}\right\rangle _{L^{2}(\Omega)}^{2}>0.
  \end{align*}
  However, on the other hand, 
  \begin{align*}
    \int_{\Omega}\phi_{j_{1}}(\vx)k(\vx,\vx')\phi_{j_{1}}(\vx')\diff{\vx}\diff{\vx'} & =\int_{\Omega}\phi_{j_{1}}(\vx)\sum_{j=1}^{\infty}\sigma_{j}\phi_{j}(\vx)\phi_{j}(\vx')\phi_{j_{1}}(\vx')\diff{\vx}\diff{\vx'}\\
    &= \sum_{j}\sigma_{j}\left\langle \phi_{j},\phi_{j_{1}}\right\rangle _{L^{2}(\Omega)}^{2}= \sigma_{j_{1}}=0,
  \end{align*}
  which leads to an contradiction. Therefore, $\nabla_{\vtheta}h(\cdot,\vtheta_{0})^{\T}\vtheta'\in H_{k}(\Omega)$.
\end{proof}
\begin{thm}
  \label{thm:min kernel norm}Let $\vtheta$ be the solution of problem
  \eqref{eq:min norm}, then $h(\vx,\vtheta)$ uniquely solves the optimization
  problem 
  \begin{equation}
    \min_{h-h_\mathrm{ini}\in H_{k}(\Omega)}\norm{h-h_\mathrm{ini}}_{k}\quad \mathrm{s.t.}\quad \vh(\mX)=\vY,\label{eq:minKnorm}
  \end{equation}
  where $h_\mathrm{ini}=h(\vx,\vtheta_{0})$ and the constraints $\vh(\mX)=\vY$
  are in the sense of trace \citep{evans2010partial} if the Hilbert space $H_k(\Omega)$ is good enough for the existence of trace of $\vh$ at $\mX$.
\end{thm}

\begin{proof}
  By Eq. \eqref{eq:linear} $h(\vx,\vtheta)-h_\mathrm{ini}=\nabla_{\vtheta}h(\vx,\vtheta_{0})\cdot\left(\vtheta-\vtheta_{0}\right)$.
  By Lemma \ref{lem:thetatoh}, $h(\cdot,\vtheta)-h_\mathrm{ini}\in H_{k}(\Omega)$.
  For any $h-h_\mathrm{ini}\in H_{k}(\Omega)$, by Lemma \ref{lem:htotheta},
  for $\vtheta'=\left\langle h-h_\mathrm{ini},\nabla_{\vtheta}h(\cdot,\vtheta_{0})\right\rangle _{k}$,
  $h-h_\mathrm{ini}=\nabla_{\vtheta}h(\vx,\vtheta_{0})^{\T}\vtheta'$. Then 
  \begin{align*}
    \norm{h-h_\mathrm{ini}}_{k} 
    &= \norm{\nabla_{\vtheta}h(\cdot,\vtheta_{0})^{\T}\vtheta'}_{k}\\
    &= \sqrt{\left\langle h-h_\mathrm{ini},\nabla_{\vtheta}h(\cdot,\vtheta_{0})^{\T}\vtheta'\right\rangle _{k}}\\
    &= \sqrt{\left\langle h-h_\mathrm{ini},\nabla_{\vtheta}h(\cdot,\vtheta_{0})\right\rangle _{k}^{\T}\vtheta'}\\
    &= \sqrt{\vtheta'^{\T}\vtheta'}= \norm{\vtheta'}_{2}.
  \end{align*}
  By Problem \eqref{eq:min norm}, for any $\vtheta_{1}\neq\vtheta-\vtheta_{0}$
  that satisfies $\vh(\mX,\vtheta_{1}+\vtheta_{0})=\vY$, we have
  \begin{equation*}
    \norm{\vtheta_{1}}_{2}>\norm{\vtheta-\vtheta_{0}}_{2}.
  \end{equation*}
  Then, for problem \eqref{eq:minKnorm}, for any $h_{1}$ satisfying
  $h_{1}-h_\mathrm{ini}\in H_{k}(\Omega)$, $h_{1}(\mX)=\vY$ and
  $h_{1}(\vx)\neq h(\vx,\vtheta)$, let $\vtheta_{1}=\left\langle h_{1}-h_\mathrm{ini},\nabla_{\vtheta}h(\cdot,\vtheta_{0})\right\rangle _{k}$.
  Clearly, $\vtheta_{1}\neq\vtheta-\vtheta_{0}$, which leads to
  \begin{equation*}
  \norm{h_{1}-h_\mathrm{ini}}_{k}=\norm{\vtheta_{1}}_{2}>\norm{\vtheta-\vtheta_{0}}_{2}=\norm{h(\vx,\vtheta)-h_\mathrm{ini}}_{k}.
  \end{equation*}
  Therefore $h(\vx,\vtheta)$ uniquely solves problem \eqref{eq:minKnorm}.
\end{proof}
Now, we obtain the equivalence between the long time solution of dynamics
\eqref{eq:gdh-h2} and the solution of optimization problem \eqref{eq:minKnorm-1}
as follows. 
\begin{cor}
  \label{cor:equiv-h}Let $h(\vx,t)$ be the solution of dynamics
  \begin{equation}
    \frac{\D }{\D t}h(\vx,t)=-\vk(\vx,\mX)\nabla_{\vh(\mX,t)}\dist\left(\vh(\mX,t),\vY\right),\label{eq:gdh-h2}
  \end{equation}
  with $h(\vx,0)=h(\vx,\vtheta_{0})$ for certain $\vtheta_{0}$. Then $h(\vx,\infty)$
  uniquely solves optimization problem 
  \begin{equation}
    \min_{h-h_\mathrm{ini}\in H_{k}(\Omega)}\norm{h-h_\mathrm{ini}}_{k},\quad s.t.,\quad \vh(\mX)=\vY.\label{eq:minKnorm-1}
  \end{equation}
\end{cor}

\begin{proof}
  Notice that dynamics \eqref{eq:gdh-h2} is the same as dynamics \eqref{eq:gfh}
  obtained from \eqref{eq:gdh}. Therefore, for $h(\vx,0)=h(\vx,\vtheta_{0})$,
  $h(\vx,t)=h(\vx,\vtheta(t))$ where $\vtheta(t)$ is the solution of dynamics
  \eqref{eq:gdh} with   $\vtheta(0)=\vtheta_{0}$. By
  Theorem \ref{thm:equivalence} and \ref{thm:min kernel norm}, $h(\vx,\infty)=h(\vx,\vtheta(\infty))$
  uniquely solves dynamics \eqref{eq:minKnorm-1}.
\end{proof}

\section{Impact of non-zero initial output \label{sec:AppendixNegative-impact-of}}

In this section, we use the above framework to show that a random
non-zero initial DNN output leads to a specific type of generalization
error. We begin with a lemma showing the linear composition property
of the final DNN outputs in the NTK regime.
\begin{lem}
  \label{linearity of solution-1} For a fixed kernel function $k\in L^2(\Omega\times\Omega)$,
  for any two training sets $\{\mX;\vY_{1}\}$ and $\{\mX;\vY_{2}\}$, where
  $\vY_{1}=(y_{1}^{(1)},\cdots,y_{n}^{(1)})^{\T}$ and $\vY_{2}=(y_{1}^{(2)},\cdots,y_{n}^{(2)})^{\T}$,
  the following linear relation holds
  \begin{equation}
    h_{k}(\cdot;0,\mX,\vY_{1}+\vY_{2})=h_{k}(\cdot;0,\mX,\vY_{1})+h_{k}(\cdot;0,\mX,\vY_{2}).
  \end{equation}
\end{lem}

\begin{proof}
  Let $h_{1}(\vx,t)$, $h_{2}(\vx,t)$ be the solutions of the gradient
  flow dynamics with respect to a MSE loss $\dist\left(\vh(\mX,t),\vY\right)=\frac{1}{2}\sum_{i=1}^{n}(h(\vx_{i},t)-y_{i})^{2}$
  \begin{align}
    \frac{\partial}{\partial t}h(\vx,t) & =-\vk(\vx,\mX)\left(\vh(\mX,t)-\vY\right)\label{eq:MSE gradient flow}
  \end{align}
  with training labels $\vY=\vY_{1}$ and $\vY=\vY_{2}$, respectively, and
  $h_{1}(\vx,0)=h_{2}(\vx,0)=h_\mathrm{ini}=0$. Then 
  \begin{align*}
    \partial_{t}\left(h_{1}+h_{2}\right) 
    &= -\vk(\cdot,\mX)\left(\vh_{1}(\mX,\vtheta(t))-\vY_{1}\right)-\vk(\cdot,\mX)\left(\vh_{2}(\mX,\vtheta(t))-\vY_{2}\right).\\
    &= -\vk(\cdot,\mX)\left[\left(\vh_{1}+\vh_{2}\right)(\mX,\vtheta(t))-\left(\vY_{1}+\vY_{2}\right)\right]
  \end{align*}
  with initial value $\left(h_{1}+h_{2}\right)(\cdot,0)=0$. Therefore
  $h_{1}+h_{2}$ solves dynamics \eqref{eq:MSE gradient flow} for $\vY=\vY_{1}+\vY_{2}$
  and $h_\mathrm{ini}=0$. Then, by Corollary \ref{cor:equiv-h}, we
  obtain
  \begin{equation*}
    h_{k}(\cdot;0,\mX,\vY_{1}+\vY_{2})=h_{1}(\cdot,\infty)+h_{2}(\cdot,\infty)=h_{k}(\cdot;0,\mX,\vY_{1})+h_{k}(\cdot;0,\mX,\vY_{2}).
  \end{equation*}
\end{proof}
Using Lemma (\ref{linearity of solution-1}), we obtain the following
quantitative relation between the solution with zero initial output
and that with non-zero initial output.
\begin{thm}
  \label{thm:inipredict-1} (Theorem \ref{thm:inipredict} in main text)
  For a fixed kernel function $k\in L^{2}(\Omega\times\Omega)$, and
  training set $\{\mX;\vY\}$, for any initial function $h_{\mathrm{ini}}\in L^{\infty}(\Omega)$,
  $h_{k}(\cdot;h_\mathrm{ini},\mX,\vY)$ can be decomposed as
  \begin{equation}
    h_{k}(\cdot;h_\mathrm{ini},\mX,\vY)=h_{k}(\cdot;0,\mX,\vY)+h_\mathrm{ini}-h_{k}(\cdot;0,\mX,h_\mathrm{ini}(\mX)).\label{eq:networkpredictini-1}
  \end{equation}
\end{thm}

\begin{proof}
  Because $h_{k}(\cdot;h_\mathrm{ini},\mX,\vY)$ is the solution of problem
  \eqref{eq: problem min kernel norm}. Then $h_{k}(\cdot;h_{\mathrm{ini}},\mX,\vY)-h_\mathrm{ini}$
  is the solution of problem
  \begin{equation*}
  \min_{h\in H_{k}(\Omega)}\norm{h}_{k},\ s.t.,\ \vh(\mX)=\vY-h_\mathrm{ini}(\mX),
  \end{equation*}
  whose solution is denoted as $h_{k}(\cdot;0,\mX,\vY-h_\mathrm{ini}(\mX))$.
  By Lemma \ref{linearity of solution-1}, 
  \begin{equation}
    h_{k}(\cdot;0,\mX,\vY-h_\mathrm{ini}(\mX))=h_{k}(\cdot;0,\mX,\vY)-h_{k}(\cdot;0,\mX,h_\mathrm{ini}(\mX)).
  \end{equation}
  Therefore 
  \begin{align}
    h_{k}(\cdot;h_\mathrm{ini},\mX,\vY) 
    &= h_{k}(\cdot;0,\mX,\vY-h_\mathrm{ini}(\mX))+h_\mathrm{ini}\nonumber \\
    &= h_{k}(\cdot;0,\mX,\vY)+h_\mathrm{ini}-h_{k}(\cdot;0,\mX,h_\mathrm{ini}(\mX)).
  \end{align}
\end{proof}
The generalization error of DNN contributed by a random initial output
can be estimated as follows.
\begin{thm}
  \label{thm:iniextraerror-1} (Theorem \ref{thm:iniextraerror} in
  main text) For a target function $f\in L^{\infty}(\Omega)$, if $h_{\mathrm{ini}}$
  is generated from an unbiased distribution of random functions $\mu$ such
  that $\Exp_{h_\mathrm{ini}\sim \mu}h_\mathrm{ini}=0$, then the
  generalization error of $h_{k}(\cdot;h_\mathrm{ini},\mX,f(\mX))$ can be
  decomposed as follows 
\begin{align*}
  \Exp_{h_\mathrm{ini}\sim \mu}\RS\left(h_{k}(\cdot;h_\mathrm{ini},\mX,f(\mX)),f\right)
  &= \RS\left(h_{k}(\cdot;0,\mX,f(\mX)),f\right)\\
  &~~~~+\Exp_{h_\mathrm{ini}\sim \mu}\RS\left(h_{k}(\cdot;0,\mX,h_\mathrm{ini}(\mX)),h_\mathrm{ini}\right),
\end{align*}
  where $\RS(h_{k}(\cdot;h_\mathrm{ini},\mX,f(\mX)),f)=\norm{h_{k}(\cdot;h_\mathrm{ini},\mX,f(\mX))-f}_{L^{2}(\Omega)}^{2}$. 
\end{thm}

\begin{proof}
  By Theorem \ref{thm:inipredict-1},
  \begin{align}
    &~~~~\norm{h_{k}(\cdot;h_\mathrm{ini},\mX,f(\mX))-f}_{L^{2}(\Omega)}^{2}\nonumber\\
    &= \norm{\left[h_{k}(\cdot;0,\mX,f(\mX))-f\right]+\left[h_\mathrm{ini}-h_{k}(\cdot;0,\mX,h_\mathrm{ini}(\mX))\right]}_{L^{2}(\Omega)}^{2}\nonumber\\
    &= \norm{h_{k}(\cdot;0,\mX,f(\mX))-f}_{L^{2}(\Omega)}^{2}+\norm{h_\mathrm{ini}-h_{k}(\cdot;0,\mX,h_\mathrm{ini}(\mX))}_{L^{2}(\Omega)}^{2}\nonumber\\
    &~~~~+2\left\langle h_{k}(\cdot;0,\mX,f(\mX))-f,h_\mathrm{ini}-h_{k}(\cdot;0,\mX,h_\mathrm{ini}(\mX))\right\rangle _{L^{2}(\Omega)}.
  \end{align}
  Because $\Exp_{h_\mathrm{ini}\sim \mu}h_{\rm ini}=0$, by Lemma \ref{linearity of solution-1},
  $\Exp_{h_\mathrm{ini}\sim \mu}\left[h_{k}(\cdot;0,\mX,h_\mathrm{ini}(\mX))\right]=\left[h_{k}(\cdot;0,\mX,\Exp_{h_\mathrm{ini}\sim \mu}h_\mathrm{ini}(\mX))\right]=0$,
  \begin{align}
    &~~~~\Exp_{h_\mathrm{ini}\sim \mu}\left\langle h_{k}(\cdot;0,\mX,f(\mX))-f,h_\mathrm{ini}-h_{k}(\cdot;0,\mX,h_\mathrm{ini}(\mX))\right\rangle _{L^{2}(\Omega)}\nonumber \\
    &= \left\langle h_{k}(\cdot;0,\mX,f(\mX))-f,\Exp_{h_\mathrm{ini}\sim \mu}\left[h_\mathrm{ini}-h_{k}(\cdot;0,\mX,h_\mathrm{ini}(\mX))\right]\right\rangle _{L^{2}(\Omega)} 
    = 0.
  \end{align}
  Then we obtain
\begin{align*}
    \Exp_{h_\mathrm{ini}\sim \mu}\RS\left(h_{k}(\cdot;h_\mathrm{ini},\mX,f(\mX)),f\right)
    &= \RS\left(h_{k}(\cdot;0,\mX,f(\mX)),f\right)\\
    &~~~~+\Exp_{h_\mathrm{ini}\sim \mu}\RS\left(h_{k}(\cdot;0,\mX,h_\mathrm{ini}(\mX)),h_\mathrm{ini}\right)
\end{align*}
\end{proof}

\section{AntiSymmetrical Initialization trick (ASI) \label{sec:AppendixASI}}

We design an AntiSymmetrical Initialization trick (ASI) which can
make the initial output zero but also keep the kernel invariant. Let
$h_{i}^{[l]}$ be the output of the $i$th node of the $l$th layer
of a $L$ layer DNN. Then, $h_{i}^{[l]}(\vx)=\sigma_{i}^{[l]}(\mW_{i}^{[l]}\cdot \vh^{[l-1]}(\vx)+b_{i}^{[l]})$,
for $i=1,\cdots,m_{l}$. For the $i$th neuron of the output layer
of DNN $h_{i}^{[L]}(\vx)=\mW_{i}^{[L]}\cdot \vh^{[L-1]}+b_{i}^{[L]}.$ After
initializing the network with any conventional method, we obtain $h^{[L]}(\vx,\vtheta(0))$,
where 
\begin{equation*}
  \vtheta(0)=(\mW^{[L]}(0),\vb^{[L]}(0),\mW^{[L-1]}(0),\vb^{[L-1]}(0),\cdots,\vb^{[1]}(0)).
\end{equation*}
The ASI for general loss functions is to consider a new DNN with output
$h_{\mathrm{ASI}}(\vx,\vTheta(t))=\frac{\sqrt{2}}{2}h^{[L]}(\vx,\vtheta(t))-\frac{\sqrt{2}}{2}h^{[L]}(\vx,\vtheta'(t))$
where $\vTheta=(\vtheta,\vtheta')$, $\vTheta$ is initialized such that
$\vtheta'(0)=\vtheta(0)$. We will prove that ASI trick
eliminates the nonzero prior without changing the kernel $k$.
\begin{thm}
  (Theorem \ref{thm:kernelinv} in main text) By applying trick ASI
  to any DNN $h(\vx,\vtheta(t))$ initialized by $\vtheta(0)=\vtheta_{0}$
  such that $h_{\mathrm{ini}}=h(\vx,\vtheta_{0})\neq0$, we obtain a new
  DNN $h_{\mathrm{ASI}}(\vx,\vTheta(t))=\frac{\sqrt{2}}{2}h(\vx,\vtheta_{1}(t))-\frac{\sqrt{2}}{2}h(\vx,\vtheta_{2}(t))$
  ($\vTheta=(\vtheta_{1},\vtheta_{2})$) with initial value $\vtheta_{1}(0)=\vtheta_{2}(0)=\vtheta_{0}$.
  Then, for any general loss function $\dist$, in the NTK regime, the
  evolution of both $h(\vx,\vtheta(t))$ and $h_{\mathrm{ASI}}(\vx,\vTheta(t))$
  under gradient flow of both DNNs follows kernel dynamics
  \begin{equation}
    \partial_{t}h'=-\vk(\cdot,\mX)\nabla_{\vh(\mX,t)}\dist\left(\vh'(\mX,t),\vY\right),\label{eq:samedy-1}
  \end{equation}
  with initial value $h'(\cdot,0)=h_{\mathrm{ini}}$ and $h'(\cdot,0)=0$,
  respectively, where $\{\mX;\vY\}$ is the training set, $k(\vx,\vx')=k_{\vtheta_0}(\vx,\vx')=\nabla_{\vtheta}h(\vx,\vtheta_{0})^{\T}\nabla_{\vtheta}h(\vx',\vtheta_{0})$.
\end{thm}

\begin{proof}
  Clearly, $h(\cdot,\vtheta(t))$ under gradient flow follows dynamics
  \eqref{eq:samedy-1} with initial function $h'(\cdot,0)=h_{\mathrm{ini}}$.
  For the evolution of $h_{\mathrm{ASI}}(\vx,\vTheta(t))$, it is easy
  to see that it follows dynamics \eqref{eq:samedy-1} with initial
  function $h'(\cdot,0)=0$ if and only if $h_{\mathrm{ASI}}(\cdot,\vTheta(0))=0$
  and $k_{\vTheta_0}=k_{\vtheta_0}$. By the definition of $k_{\vTheta(0)}$,
  \begin{align*}
    k_{\vTheta_0}(\vx,\vx') & =\nabla_{\vTheta}h_{\mathrm{ASI}}(\vx,\vTheta(0))\cdot\nabla_{\vTheta}h_{\mathrm{ASI}}(\vx',\vTheta(0))\\
    &= \frac{1}{2}\nabla_{\vtheta_{1}}h(\vx,\vtheta_{1}(0))\cdot\nabla_{\vtheta_{1}}h(\vx',\vtheta_{1}(0))+\frac{1}{2}\nabla_{\vtheta_{2}}h(\vx,\vtheta_{2}(0))\cdot\nabla_{\vtheta_{2}}h(\vx',\vtheta_{2}(0))\\
    &= \frac{1}{2}\nabla_{\vtheta}h(\vx,\vtheta_{0})\cdot\nabla_{\vtheta}h(\vx',\vtheta_{0})+\frac{1}{2}\nabla_{\vtheta}h(\vx,\vtheta_{0})\cdot\nabla_{\vtheta}h(\vx',\vtheta_{0})\\
    &= \nabla_{\vtheta}h(\vx,\vtheta_{0})\cdot\nabla_{\vtheta}h(\vx',\vtheta_{0})= k_{\vtheta(0)}.
  \end{align*}
  Moreover, 
  \begin{align*}
    h_{\mathrm{ASI}}(\vx,\vTheta(0)) & =\frac{\sqrt{2}}{2}h(\vx,\vtheta_{1}(0))-\frac{\sqrt{2}}{2}h(\vx,\vtheta_{2}(0))\\
    &= \frac{\sqrt{2}}{2}h(\vx,\vtheta_{0})-\frac{\sqrt{2}}{2}h(\vx,\vtheta_{0})= 0.
  \end{align*}
  This completes the proof.
\end{proof}

\section{``doubling trick''\label{sec:doubling-trickappendix}}

By applying the ``doubling trick'' (Note that, in \citet{chizat2018note},
there is no bias term in the last layer), we obtain a new network
with network parameters $\vtheta'=(\mW^{\prime}{}^{[L]},\mW^{\prime}{}^{[L-1]},\vb^{\prime}{}^{[L-1]},\cdots,\vb^{\prime}{}^{[1]})$
initialized as $\mW^{\prime}{}^{[L]}(0)=(\mW^{[L]}(0),-\mW^{[L]}(0))$,
$\mW^{\prime}{}^{[L-1]}(0)=(\mW^{[L-1]}(0),\mW^{[L-1]}(0))$, $\vb^{\prime}{}^{[L-1]}(0)=(\vb{}^{[L-1]}(0),\vb{}^{[L-1]}(0))$
and $\mW^{\prime}{}^{[l]}(0)=\mW^{[l]}(0)$, $\vb^{\prime}{}^{[l]}(0)=\vb^{[l]}(0)$
for any $l=1,\cdots,L-2$.

In general, the kernel can be decomposed as the summation of kernels
with respect the tangent space of parameters of the neural network
in each layer, that is
\begin{align*}
k_{\vtheta}(\vx,\vx') 
  &= \nabla_{\vtheta}h(\vx,\vtheta)\cdot\nabla_{\vtheta}h(\vx',\vtheta)\\
  &= \sum_{l=1}^{L}\left[\nabla_{\mW^{[l]}}h(\vx,\vtheta)\cdot\nabla_{\mW^{[l]}}h(\vx,\vtheta)+\nabla_{\vb^{[l]}}h(\vx,\vtheta)\cdot\nabla_{\vb^{[l]}}h(\vx,\vtheta)\right].
\end{align*}

\begin{thm}
  For the DNN initialized by $\vtheta^{\prime}$, by applying the ``doubling
  trick'', for any $m\leq L-2$, 
  \begin{equation*}
    k_{\mW^{\prime}{}^{[m]}}(\vx,\vx')=0,\quad k_{\vb^{\prime}{}^{[m]}}(\vx,\vx')=0.
  \end{equation*}
  For $m=L-1,L$, and $\vTheta=\mW{}^{[L-1]},\vb{}^{[L-1]},\mW{}^{[L]}$, 
  \begin{equation*}
    k_{\vTheta^{\prime}}(\vx,\vx')=2k_{\vTheta}(\vx,\vx'),
  \end{equation*}
\end{thm}

\begin{proof}
  For any $m\leq L-2$,
  \begin{equation*}
    \nabla_{\mW_{i,j}^{\prime[m]}}h'(\vx,\vtheta'(0))=\left(\prod_{l=m+1}^{L-1}\mW^{\prime}{}^{[l+1]}(0)\vs^{\prime}{}^{[l]}(\vx,0)\right)\mW^{\prime}{}_{j}^{[m+1]}(0)s^{\prime}{}_{j}^{[m]}(\vx,0)h^{\prime}{}_{i}^{[m-1]}(\vx,0),
  \end{equation*}
  \begin{equation*}
    \nabla_{b^{\prime}{}_{j}^{[m]}}h'(\vx,\vtheta'(0))=\left(\prod_{l=m+1}^{L-1}\mW^{\prime}{}^{[l+1]}(0)\vs^{\prime}{}^{[l]}(\vx,0)\right)\mW^{\prime}{}_{j}^{[m+1]}(0)s^{\prime}{}_{j}^{[m]}(\vx,0),
  \end{equation*}
  where $s^{\prime}{}_{i}^{[l]}(\vx,t)=s(\mW_{i}^{[l]}(t)\cdot \vh^{[l-1]}(\vx)+b_{i}^{[l]}(t))$,
  for $i=1,\cdots,m_{l}$, $s(\vx)=\frac{\D \sigma(\vx)}{\D \vx}$.
  Because 
  \begin{equation*}
    \mW^{\prime}{}^{[L]}(0)\vs^{\prime}{}^{[L-1]}(\vx,0)=\mW^{[L]}(0)\vs^{[L-1]}(\vx,0)-\mW^{[L]}(0)\vs^{[L-1]}(\vx,0)=0,
  \end{equation*}
  for any $m\leq L-2$, we obtain $\nabla_{\mW_{i,j}^{\prime[m]}}h'(\vx,\vtheta'(0))=\vec{0}$
  and $\nabla_{\vb^{\prime}{}_{j}^{[m]}}h'(\vx,\vtheta'(0))=\vec{0}$, which leads
  to $k_{\mW^{\prime}{}^{[m]}}(\vx,\vx')=0$ and $k_{\vb^{\prime}{}^{[m]}}(\vx,\vx')=0$.
  For layer $L-1$ and $L$, similarly, we have
  \begin{equation*}
    k_{\mW^{\prime}{}^{[L-1]}}(\vx,\vx')=2k_{\mW^{[L-1]}}(\vx,\vx'),
  \end{equation*}
  \begin{equation*}
    k_{\vb^{\prime}{}^{[L-1]}}(\vx,\vx')=2k_{\vb^{[L-1]}}(\vx,\vx'),
  \end{equation*}
  \begin{equation*}
    k_{\mW^{\prime}{}^{[L]}}(\vx,\vx')=2k_{\mW^{[L]}}(\vx,\vx').
  \end{equation*}
\end{proof}
Therefore, by applying the ``doubling trick'', $h'(\vx,\vtheta(0))$
is offset to $0$. However, the kernel of layers $L-1$ and $L$ doubles,
whereas the kernel of layers $m\leq L-2$ completely vanishes,
which could have large impact on the training dynamics as well as
the generalization performance of DNN output.


\end{document}